\documentclass[letterpaper, 10 pt, conference]{ieeeconf}  

\IEEEoverridecommandlockouts                              

\usepackage[utf8]{inputenc} 
\usepackage[T1]{fontenc}    
\usepackage{hyperref}       
\usepackage{url}            
\usepackage{booktabs}       
\usepackage{amsfonts}       
\usepackage{nicefrac}       
\usepackage{microtype}      
\usepackage{xcolor}         

\usepackage{amsmath} 
\usepackage{amssymb}  
\usepackage{diagbox}
\usepackage{graphicx}
\usepackage{algorithm2e}
\usepackage{balance}
\usepackage{caption}
\usepackage{subcaption}
\usepackage[free-standing-units=true]{siunitx}

\newtheorem{proposition}{Proposition}

\usepackage{listings} 
\title{\LARGE \bf
TT-SDF2PC: Registration of Point Cloud and Compressed SDF
Directly in the Memory-Efficient Tensor Train Domain
}

\author{Alexey I. Boyko$^{*,1,2}$, Anastasiia Kornilova$^{*,2}$, Rahim Tariverdizadeh$^{2}$, \\
Mirfarid Musavian$^{2}$,  Larisa Markeeva$^{2}$, Ivan Oseledets$^{2}$, Gonzalo Ferrer$^{2}$
\thanks{$^{*}$ Indicates equal contributions from both authors.}
\thanks{ $^{1}$ AIRI Institute,
$^{2}$ Skolkovo Institute of Science and Technology
(Skoltech), Center for AI Technology (CAIT). {\tt\small @skoltech.ru} }%
}

\begin{document}

\maketitle
\thispagestyle{empty}
\pagestyle{empty}

\begin{abstract}
This paper addresses the following research question: ``can one compress a detailed  3D representation and use it directly for point cloud registration?''. Map compression of the scene can be achieved by the tensor train (TT) decomposition of the signed distance function (SDF) representation. It regulates the amount of data reduced by the so-called TT-ranks.

Using this representation we have proposed an algorithm, the TT-SDF2PC, that is capable of directly registering a PC to the compressed SDF by making use of efficient calculations of its derivatives in the TT domain, saving computations and memory.
We compare TT-SDF2PC with SOTA local and global registration methods in a synthetic dataset and a real dataset and show on par performance while requiring significantly less resources.
\end{abstract}


\section{Introduction}

Map compression allows to significantly reduce the {\em memory requirements} to store 3D scenes or objects, which is mandatory for large-scale mapping and localization \cite{cadena2016past} and the deployment of robots with low-power computational devices operating within limited resources.
The reasons to reduce the amount of data to represent a scene are numerous and evident, and even more to directly work with such a reduced representation. Consequently, this paper addresses the following research question:
{``Can one compress a map {\em and} operate directly on this compressed domain, without uncompressing?''}. The particular task that we analyze in this paper is point cloud registration, which is ubiquitous in perception in robotics: motion estimation, localization, mapping, 3D reconstruction, etc.

In order to achieve this goal, first we need to show that compression achieves a significant reduction in the required memory to store. Point clouds (PC), 3D grids, voxels, meshes, signed distance functions (SDF) or even neural function approximators are some representations for these 3D environments, each of them having strengths and weaknesses.
Signed distance function and its truncated version (TSDF) are a popular choice for implicit shape representations, because of its high expressiveness. Therefore, they find active applications in 3D reconstruction problems, both in classic~\cite{Curless1996, KinectFusion2011, Bylow2013} and deep learning approaches~\cite{dist, deepsdf2019, sdfdiff, Niemeyer}.

\begin{figure}[t]
\centering
\includegraphics[width=.48\textwidth]{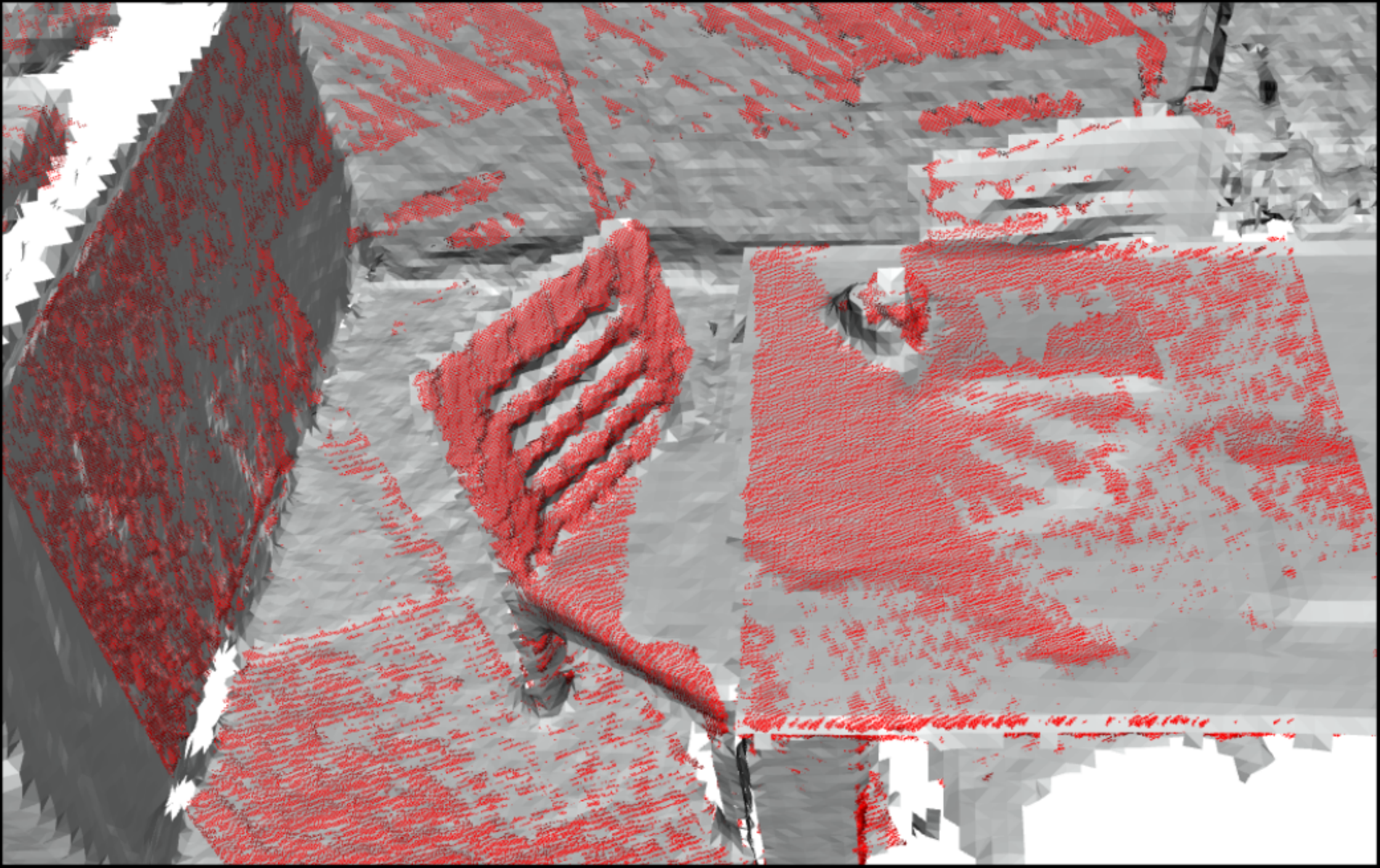}
\caption{Result of the TT-SDF2PC registration algorithm that operates directly on compressed SDF map in the \emph{7-scenes-redkitchen} environment from 3DMatch~\cite{shotton20137scenes}. The compressed TT-SDF model (gray) is aligned with a depth map of the scene (red). }
\label{fig_teaser}
\end{figure}

Despite the numerous advantages of SDF for representing 3D objects and scenes, the storage requirements in the naive setting of voxel arrays makes them barely usable in scenarios when a high-resolution 3D shape is needed. Deep learning approaches such as DeepSDF \cite{deepsdf2019}, on the other hand, in principle allow for compact latent representation of SDFs, but they are limited to object-centric datasets  which make them unusable in the robotics framework.
In our previous work, TT-TSDF \cite{tttsdf}, we researched on map compression and came with a tensor train (TT) based solution that only requires around 3 MB to store a high-resolution $512^3$ voxelized TSDF model.

The second objective is related to the success of the registration task. The compress rate is irrelevant if the degree of accuracy of the registration is low. 
We propose an algorithm for point cloud registration with respect to a SDF scene that operates directly on the tensor-compressed implicit representation domain.
This poses the main challenge of the present work, where we will show the validity of our method and compare with other state of the art works on point cloud registration, both local and global methods.
The benchmarks chosen are ModelNet \cite{modelnet}, popular among the graphics community, and 3Dmatch \cite{zeng20173dmatch}, a compendium of datasets, real scenes observed by RGBD camera.

Contributions of this paper are the following:
\begin{enumerate}
    \item Direct registration of a PC observation and the tensor-compressed high-resolution representation of the SDF scene.
    \item Evaluation of our algorithm with SOTA PC alignment methods, showing the validity on graphic models and indoors scenes showing almost the same order of accuracy in registration w.r.t local registration methods, providing the advantage of a detailed surface representation with the same order of memory consumption as sparse point clouds on huge scenes.
\end{enumerate}

\section{Related Work}

On its simplest form, the problem of point cloud registration has an elegant closed-form solution \cite{arun1987,horn1987}, which requires the strong assumption of known point associations.
Iterative closest point (ICP) \cite{besl1992} overcomes the associations limitation by following a greedy approach for finding pairs of points, in what we can consider a {\em local} registration method. This popular approach spanned multiple research directions \cite{rusik2001ICP_efficient, zhang1994} including point-to-point, point-to-line \cite{censi2008}, point-to-plane \cite{chen1991}, and other variants emulating planes on different manners \cite{segal2009,serafin2015,ferrer2019}.

An interesting variant of {\em local} registration is by a volumetric representation, the TSDF (also sometimes called projectional SDF) stored in a uniform 3D grid.
Related works \cite{sdftracker, Bylow2013,sdf2sdf, sdf2sdf2} use finite-difference derivatives coupled with optimization methods for point-to-implicit and implicit-to-implicit alignment.
Our proposed method TT-SDF2PC, is a variant of {\em local} point-to-implicit registration method, with a significant difference: the domain for registration is directly the compressed representation, which saves the need to uncompress the map, and alleviates the memory requirements.

How to generate these volumetric representations from point clouds is an important part of this work.
In general, fusion algorithms \cite{Curless1996} allow to generate the required volumetric representation. Variants such as KinectFusion \cite{KinectFusion2011},
VDBFusion \cite{vizzo2022sensors} from OpenVDB \cite{museth2013vdb,museth2013openvdb} allow this operation.
SDF has also been shown to be useful for other robotics problems \cite{oleynikova2017voxblox}.

{\em Global} registration methods provide a solution without need of an initial condition, therefore its convergence basin is superior to those from the {\em local} methods. Robust associations \cite{gelfand2005robust} find a global solution with relative good performance. This problem is inevitably entangled with the quality of the point features to find good correspondences. To this end,
early works on 3D local features were investigated, such as the FPFH \cite{rusu2009fast} or SHOT \cite{salti2014shot}. 
This first generation of features are heuristic methods considered to be {\em local} and often are not discriminative enough.

Modern machine learning, in particular neural architectures for point clouds, 
yield outstanding results in some benchmarks. Examples include LocNet\cite{yin2018locnet}, deep closest point \cite{wang2019DCP} making use of transformers,
learning multi-view \cite{gojcic2020minkowsky} or
multi-scale architecture and unsupervised transfer learning MSSVConv \cite{horache2021mssvconv}.
In addition to point features, researchers have investigated more human-understandable representations to solve the registration problem, such as
semantics \cite{zaganidis2018integrating} or segments \cite{dube2020segmap}.

Direct improvements over the optimization part also bring improvements. Fast global registration (FGR) \cite{zhou2018FGR} use graduated non-convexity with robust estimators. TEASER \cite{yang2020teaser} filters incorrect associations by a sequence of hypothesis test.

On compression of 3D scenes, meshes are a highly compact representation, but its calculation is complex, not exempt of problems. Point clouds, can be decimated or reduced to voxels, which brings a reduction in the number of points, but its expressiveness of details is questionable.
Grid-based representations, show an overall order of magnitude for compression around 5-10 times, which roughly corresponds to their inherent sparsity.
In our previous work TT-SDF \cite{tttsdf}, the authors propose a compression algorithm of TSDF representation  by several orders of magnitude by maintaining a reasonably model consistency, which to the best of our knowledge, its results are way ahead from other related works.
Regardless of the compression chosen, it is unclear how one can use it for PC registration.
In this paper, we provide a direct usage, by proposing an algorithm for point cloud registration that combines compactness and precision of SDF-based registration.

\section{Point Cloud Registration}

\subsection{Point Cloud to Point Cloud Registration (PC2PC)}

Let us have two 3D point clouds of $I$ points:  $p_i, p_i' \in \mathbb{R}^3, i=1 \ldots I$. We assume that there is a rigid body transformation $T \in SE(3)$, consisting of a 3D rotation $R \in SO(3)$ and translation by a vector $t \in \mathbb{R}^3$, acting on the entire point cloud: $ p_i' = R p_i + t. $

Alternatively, it can be written as:
\begin{equation*}
     \tilde{p}_i '= T \tilde{p}_i,
\end{equation*}
where $\tilde{p} = \left[ \begin{matrix} p \\ 1 \end{matrix} \right]$ is an homogeneous vector and $T=\left[ \begin{matrix} R & t \\ 0 & 1 \end{matrix} \right]$ is an element of the $SE(3)$ group. The rotation matrix $R \in SO(3)$ is considered to be parametrized by the axis–angle representation.

The problem of point cloud alignment can be formulated as a function to minimize:
\begin{equation}
    T^{*} = \text{arg}\,\min_{T} \sum_{i}^{I} ||T\cdot \tilde{p}_i - \tilde{p}'_i||^2 =  \text{arg}\,\min_{T} f(T),
    \label{eq:loss}
\end{equation}
assuming known correspondences between points.

Therefore, we want to find a minimum for $f()$ by minimizing with respect to $T$, which requires to optimize over the manifold $SE(3)$.
To this end, we can define a retraction \cite{absil2009} function as
\begin{equation}
R_T(\xi) = \exp( \xi^{\land} ) T = \mathrm{Exp}( \xi) T,
\end{equation}
where the exponent indicates the exponent map \cite{lynch2017} and $\xi \in \mathbb{R}^6$ are local coordinates of the Lie algebra of $SE(3)$ around the identity.

\begin{figure*}[h!] 
\includegraphics[width=.19\textwidth]{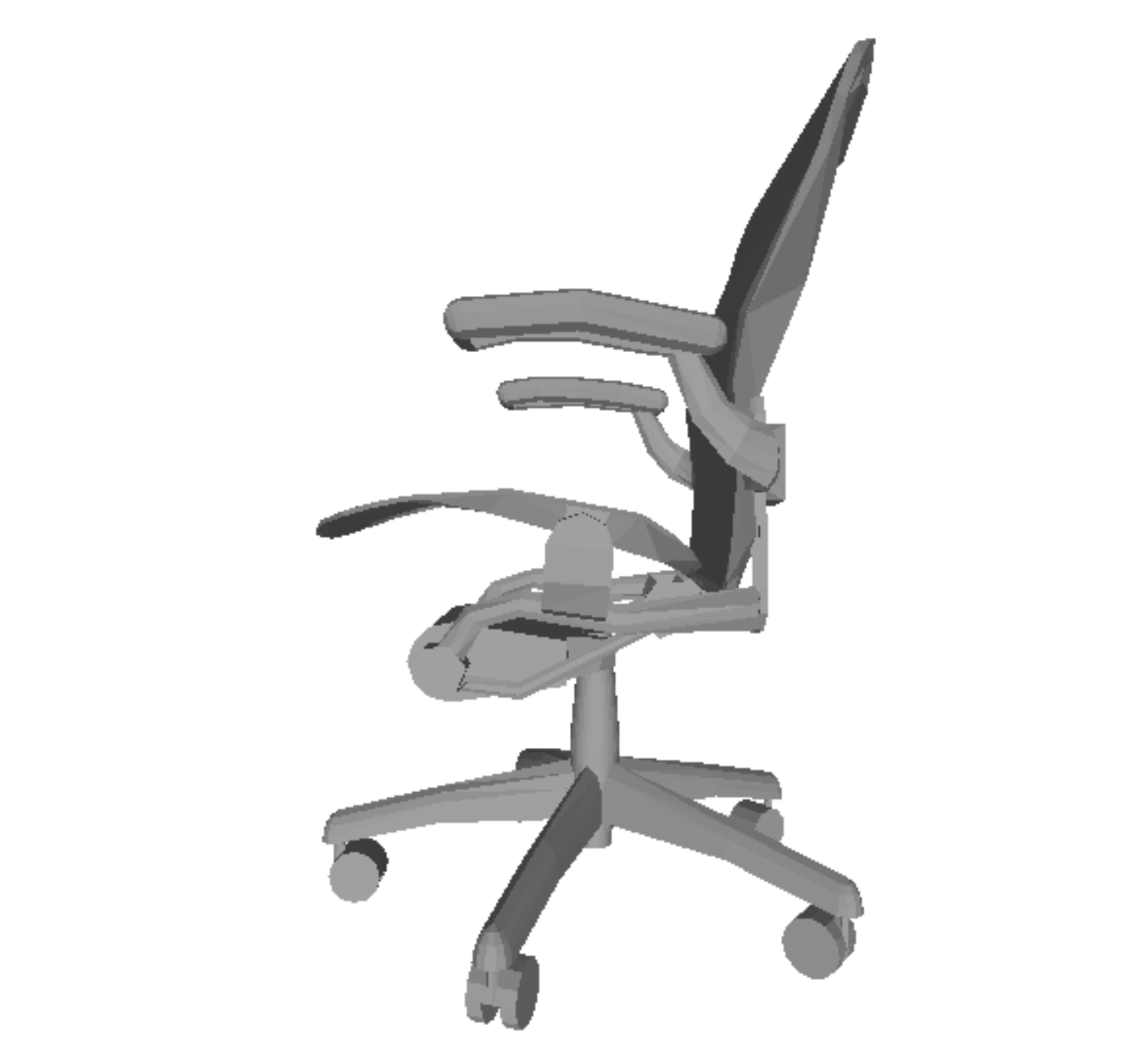} 
\includegraphics[width=.19\textwidth]{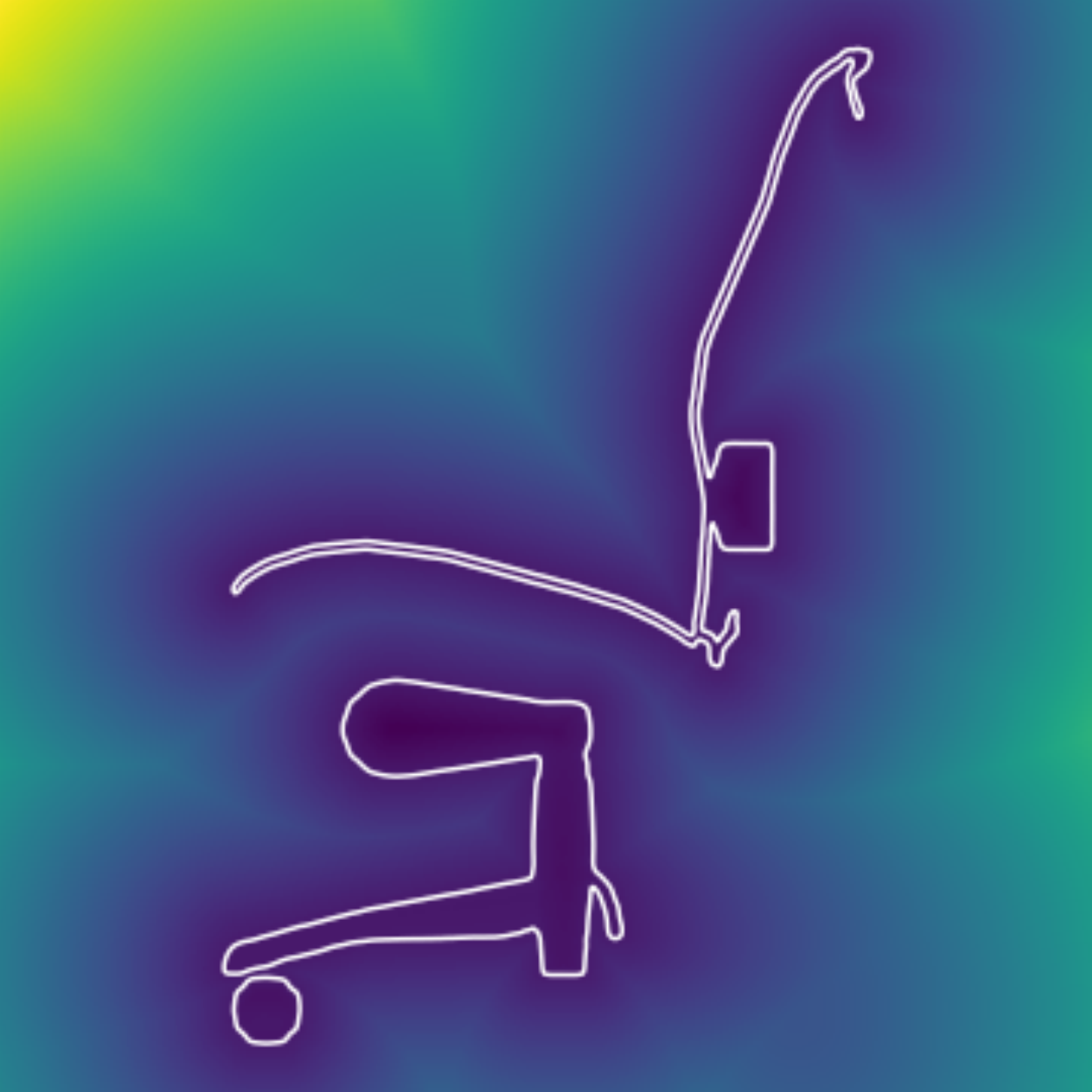} 
\includegraphics[width=.19\textwidth]{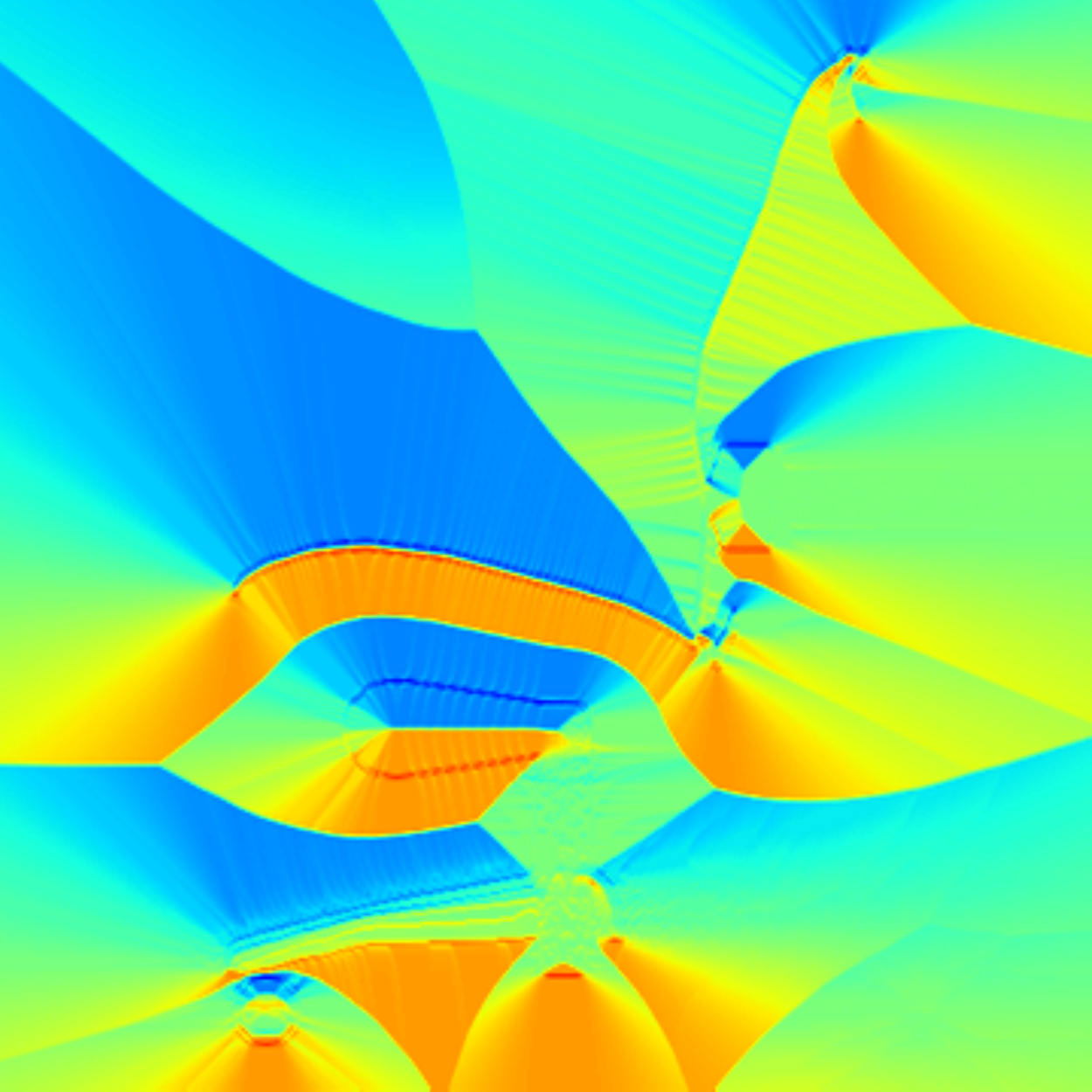}
\includegraphics[width=.19\textwidth]{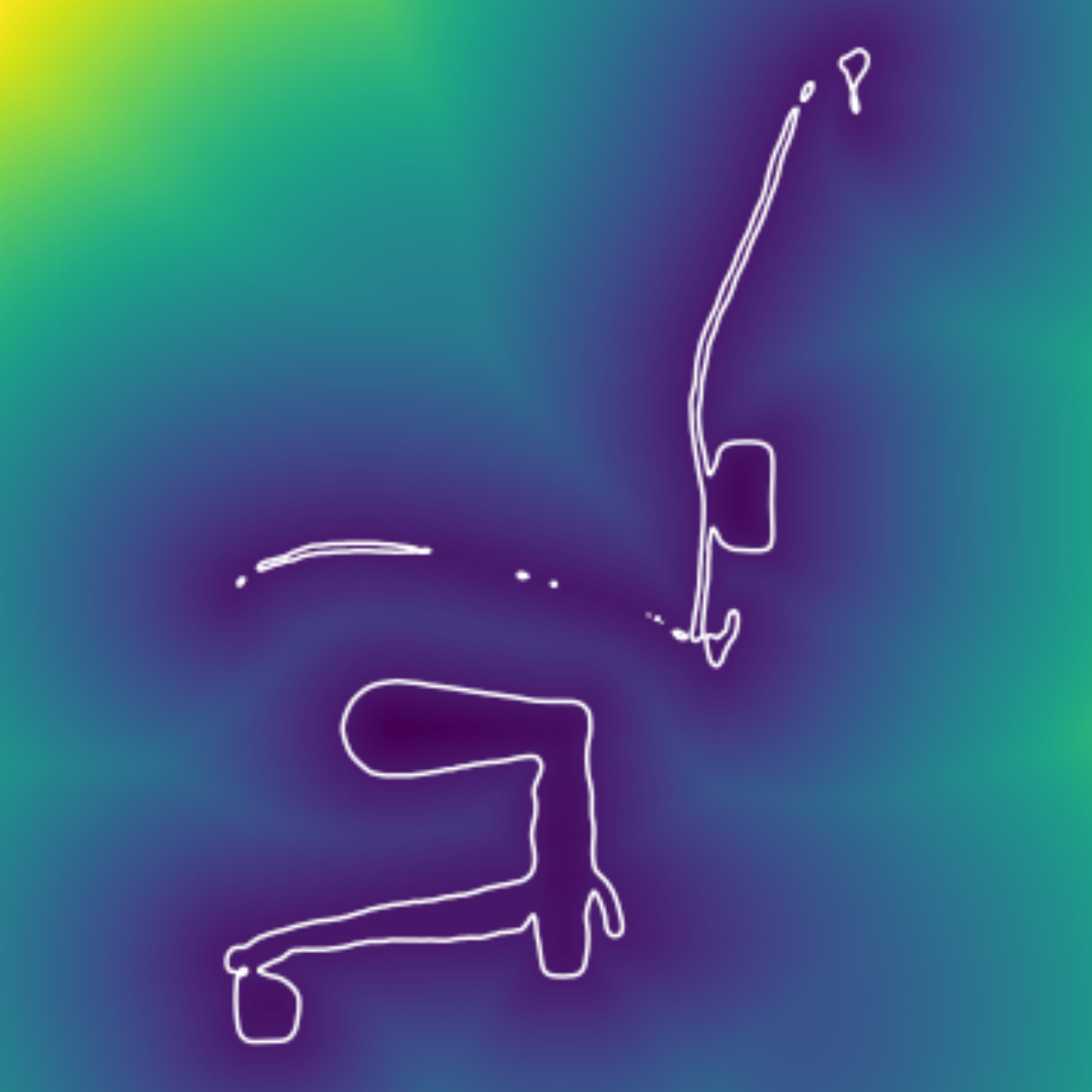}
\includegraphics[width=.19\textwidth]{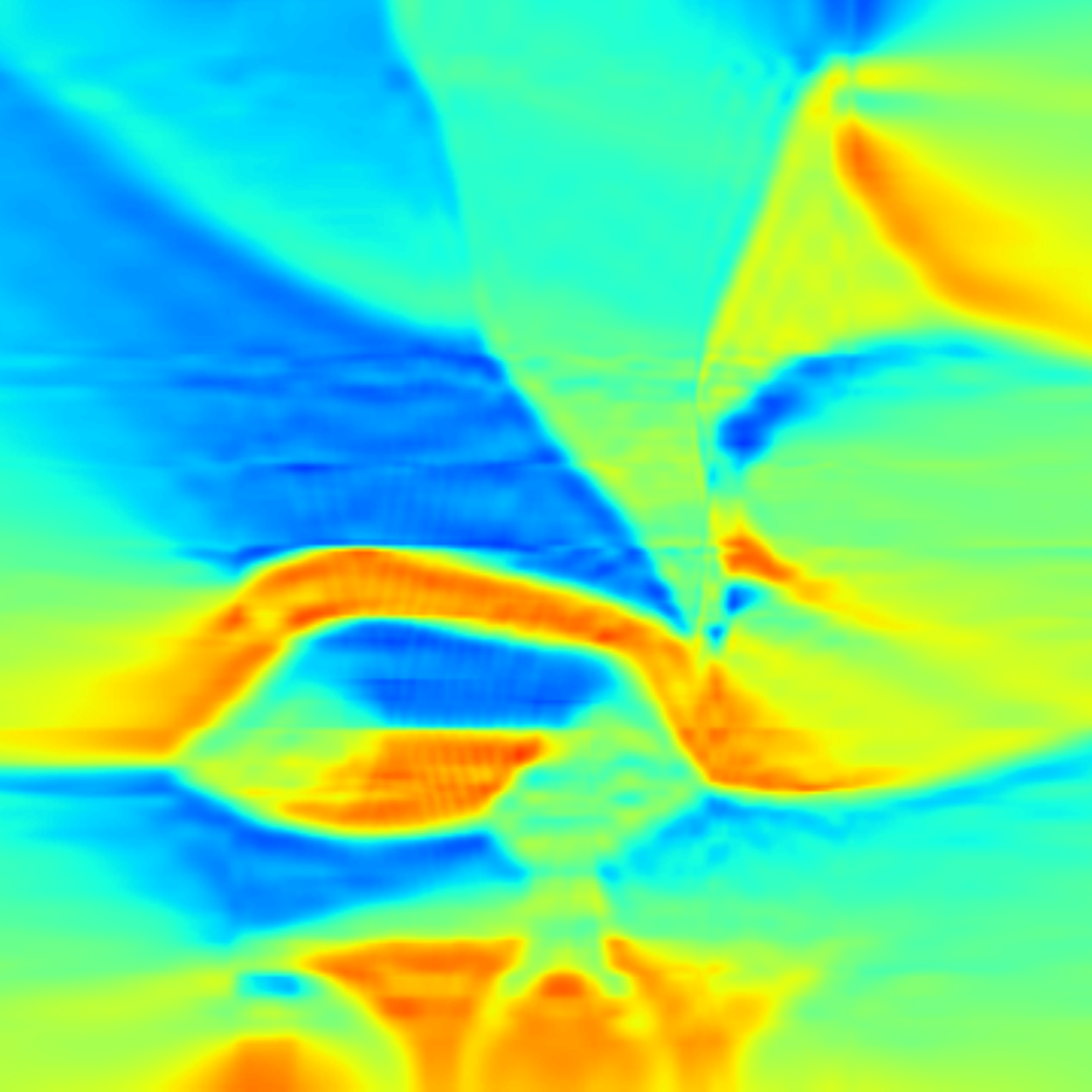} 
\caption{(a) Mesh model (b) True SDF, (c) its gradient $\nabla_x SDF(x)$,  (d) TT-SDF ($R$=20), (e) its gradient $\nabla_x SDF(x)$. Figures (b)-(e) are taken in the middle crossection of the chair.}
 \label{fig:grads}
\end{figure*}

As a result, one can do function composition $f_T = f \circ R_T : \mathbb{R}^6 \to \mathbb{R}$ now this function accepting standard calculation of directional derivatives, necessary to optimize in an iterative fashion:

\begin{equation}
    T^{*} = R_T(\xi^*) = R_T(-\alpha H_\xi^{-1}\nabla_\xi f_T(\xi)),
    \label{eq_update}
\end{equation}
which is a second order optimization method, i.e., involving the inverse of the Hessian $H_\xi$. An alternative we will not use here, but adopted by many others, is the first order update $\xi^* = -\alpha' \nabla_\xi f_T(\xi)$, the gradient descent. In the sections below we will formulate an alternative to cost function to the alignment problem, with a different cost function $f_T$, but entirely based on (\ref{eq_update}).

\subsection{SDF to Point Cloud Registration (SDF2PC)}

This variant of the problem requires an SDF and a point cloud. The goal is to {\em transform} the point cloud to the surface of the object, that is, the zero set level of the SDF such that $SDF(x)=0$, where $x\in \mathbb{R}^3$ is a point in the space. The MSE loss function from the \eqref{eq:loss} and \eqref{eq_update} becomes

\begin{equation}
    \min\limits_\xi  f_T(\xi) = \min\limits_\xi  \sum\limits_{i}^I ||SDF\left( T \cdot \tilde p_i  \right)||^2.
\end{equation}

The solution to the least squares problem above requires the calculation of the Jacobian:

\begin{align}
    J_i (x) & = \nabla_\xi SDF(x) = \nabla_x SDF(x) \cdot \frac{\partial x}{\partial \xi} =\nonumber \\
    & = \nabla_x SDF(x) \cdot  \left( -x^{\land} | I_{3 \times 3} \right)
\end{align}

\begin{equation}
    \nabla_\xi f_T(\xi) = \sum\limits_i^I J_i([T \cdot \tilde p_i]) \cdot SDF([T \cdot \tilde p_i]).
\end{equation}

The operator $[\cdot]$ indicates a rounding of the transformed point to the nearest voxel centroid in the SDF. We will discuss on Sec. \ref{sec:exp} the effect on the accuracy achieved.

The Hessian from the equation \eqref{eq_update} at every iteration is computed  following the Gauss-Newton Hessian approximation:
\begin{equation}
    H_\xi \approx \sum\limits_i^I J_i(p_i) J_i^{\top}(p_i).
\end{equation}

The optimization stops when, in two consecutive steps, the the update of the pose falls below a threshold.

\section{Tensor Train decomposition}

Low-rank tensor decompositions have been used for lossy compression of multi-dimensional data for almost a century \cite{first_lowrank}. There are multiple ways to decompose a tensor into its low-rank representation. In this paper we will focus on the Tensor Train decomposition due to its optimal trade-off between time-of-access and compression ratio for a given fixed error tolerance.

The decomposition is defined as follows. Let us consider a $d$-dimensional array $F(i_1,\ldots,i_d)$ with sizes of modes $N_1, \ldots ,N_d$. Then the following representation of this array is called a Tensor Train:
\begin{align}
\label{eq:tt}
\begin{split}
F(i_1,i_2,\ldots,i_d) = 
\sum\limits_{m_1,m_2,\ldots,m_{d-1}}^{r_1,r_2,\ldots,r_{d-1}} G^{(1)} (i_1, m_1) \cdot \\ \cdot G^{(2)}(m_1,i_2, m_2) \cdot \ldots \cdot G^{(d)} (m_{d-1},i_d),
\end{split}
\end{align}
where arrays $G^{(k)}$ are called TT-cores, and numbers $r_k$ are called TT-ranks. In case of $d=3$ this decomposition also coincides with the TUCKER2 decomposition \cite{Kolda09}.

As can be observed directly from the formula \eqref{eq:tt}, the decomposition is not unique. For a valid low-rank TT decomposition the cores $G^{(k)}$ can be obtained using multiple algorithms, based either on linear-algebraic principles \cite{tucker, tt, ttcross}, gradient-based optimization \cite{ttals, novikov2015}, or both of those combined \cite{hooi}.

A robust, CPU-based and quasi-optimal algorithm for obtaining a TT from a known array is the TT-SVD algorithm \cite{tt}. Its key idea is to sequentially reshape the multidimensional array into rectangular matrices and decoupling the first matrix index from the second one using truncated SVD.

In the 3D case the total amount of memory required by TT is $N_1 r_1 + r_1 N_2 r_2 + N_3 r_2 = O(N R^2)$ instead of $O(N^3)$ in the full uncompressed tensor, and the compression ratio (uncompressed / compressed) is $O(\frac{N^2}{R^2})$. Therefore results are the best if the 3D data is of high resolution. In this paper we denote $R=\max(r_i)$ and $N=\max(N_i)$. The computational complexity of accessing one element of the tensor has the computational complexity of $O(R^2)$.

The full list of allowable mathematical operations on Tensor Trains and their complexities can be found in papers \cite{tt, Kolda09}.

\subsection{Functional approximation and partial derivatives in TT}
We consider the voxel SDF data compressed by a Tensor Train as a discrete functional approximation inside a 3D box $[-1,1] \times [-1,1] \times [-1,1]$, with piece-wise constant basis functions that correspond to a nearest voxel approximation. The extension to a more general set of tensor networks can be seen in \cite{tn_finite} and to differentiable basis functions can be seen in \cite{Gorodetsky2019}.

In this paper, we solve the point cloud alignment through second order gradient-based optimization on the manifold. For this, we need to have access not only to the Signed Distance Function itself at any point, but also to its partial derivatives with respect to spatial coordinates: $\nabla_xSDF(x)$. Even though formally it is defined on a discrete coordinate indices, as soon as the function is correctly represented on its discretization grid (centers of voxels), such approximation is fairly precise and has a known closed-form error bounds for partial derivatives. In this paper we will use the central difference method that has the error bound of $O(h_x^2)$ if the discretization step is $h_x$.

To compute partial derivatives for a 3D function stored in Tensor Train let us start by considering a one-dimensional function $f(x) \in C^1$. Let us project it on a regular grid with $N$ nodes: $\{x_k | x_k = x_0 + k \cdot h_x, k=0 \ldots N-1 \}$. We denote $f_k = f(x_k)$. The derivative at every non-boundary point of the grid is approximated as: 
\begin{equation}
\frac{\partial f(x)}{\partial x}\big|_{x=x_k} = \frac{f_{k+1} - f_{k-1}}{2h_x} + o(h_x^2).     
\end{equation}

For the first and the last points in the domain, the formula will be different:

$$\frac{\partial f(x)}{\partial x}\big|_{x=x_0} =  -\frac{3}{2h_x} f_0 + \frac{2}{h_x} f_1 - \frac{1}{2h_x} f_2 + o(h_x^2)$$

$$\frac{\partial f(x)}{\partial x}\big|_{x=x_{N-1}} = -\frac{1}{2h_x} f_{N-3} + \frac{2}{h_x} f_{N-2} -\frac{3}{2h_x} f_{N-1} + o(h_x^2).$$

Derivatives at all points in the grid can be computed as 
\begin{equation}
    \frac{\partial f(x)}{\partial x} \big|_{x=x_{1} \ldots x_{N-1}} = D_N f
\end{equation}
where
\begin{equation}
\label{eq:finitediff}
D_N = \frac{1}{h} \left[ 
\begin{matrix}
- 3/2 & 2 & -{1}/{2} & 0 &  \ldots & 0 \\[4pt]
{1}/{2} & 0 & -{1}/{2} & 0 & \ldots & 0 \\[4pt]
0 & {1}/{2} & 0 & -{1}/{2} & 0 \ldots & 0 \\[4pt]
\vdots \\[4pt]
0 & \ldots 0 &  {1}/{2} & 0 & -{1}/{2} & 0 \\[4pt]
0 & \ldots & 0 &  {1}/{2} & 0 & -{1}/{2} \\[4pt]
0 & \ldots & 0 & -{1}/{2} & 2 & -{3}/{2}  \\[4pt]
\end{matrix} 
\right]
\end{equation}
- a sparse matrix and $f=\left[ \begin{matrix}
f_0 & \ldots & f_{N+1}
\end{matrix} \right]^{\top}$.

Tensor decompositions are naturally compatible with finite difference operators of any order of differentiation and precision. 

\begin{proposition}
A partial derivative over $k$-th coordinate of the multi-dimensional function on a grid represented as TT can be found by applying a  (smaller) finite-difference operator to the $k$-th core over hanging coordinate index. This operation takes only $O(N_k R^2)$ operations and requires additional $O(N_k R^2)$ memory for storing.
\end{proposition}

\begin{proof}
{\rm
Let us denote $F(i_1,i_2,...,i_d)$ represented in the TT format (\ref{eq:tt}) as $F$. Then the correctness of the result may be shown using the following relation:

\begin{equation}
\begin{split}
\frac{\partial F}{ \partial x_k } \approx \hat D_k F = \left( I_1 \otimes \ldots \otimes \ D_{N_k}  \otimes \ldots \otimes I_d \right) F =\\ 
= \sum\limits_{m_1, m_2,\ldots,m_{d-1}}^{r_1,r_2,\ldots,r_{d-1}} G^{(1)} (i_1, m_1) \cdot \ldots \\
\cdot \left( D_{N_k} G^{(k)}(m_{k}, i_k, m_{k+1}) \right) \cdot \ldots \cdot G^{(d)} (m_{d-1},i_d).
\end{split}
\end{equation}

Here we denote $I_m$ - identity matrix of size $N_m \times N_m$,  $D_{N_k}$ - a sparse $N_k \times N_k$ of the one-dimensional finite difference operator as defined in  \eqref{eq:finitediff},
$\hat D_k$ - finite difference operator that acts on the entire 
$d$-dimensional grid and theoretically it is $N \times N$ matrix.

The complexity of a matrix-vector product with the sparse finite-difference operator is $O(N)$, but the contraction $D_{ij} G_{mjk}$ we have to do it for all $R^2$ combinations of free indices $m$ and $k$, thus giving  $O(N R^2)$.

To store the tensor with derivatives we would only need to store one core that has changed, and a TT core has $N R^2$ elements in the worst case.
}
\end{proof}

\begin{figure*}[h!] 
\centering
\includegraphics[width=.98\textwidth]{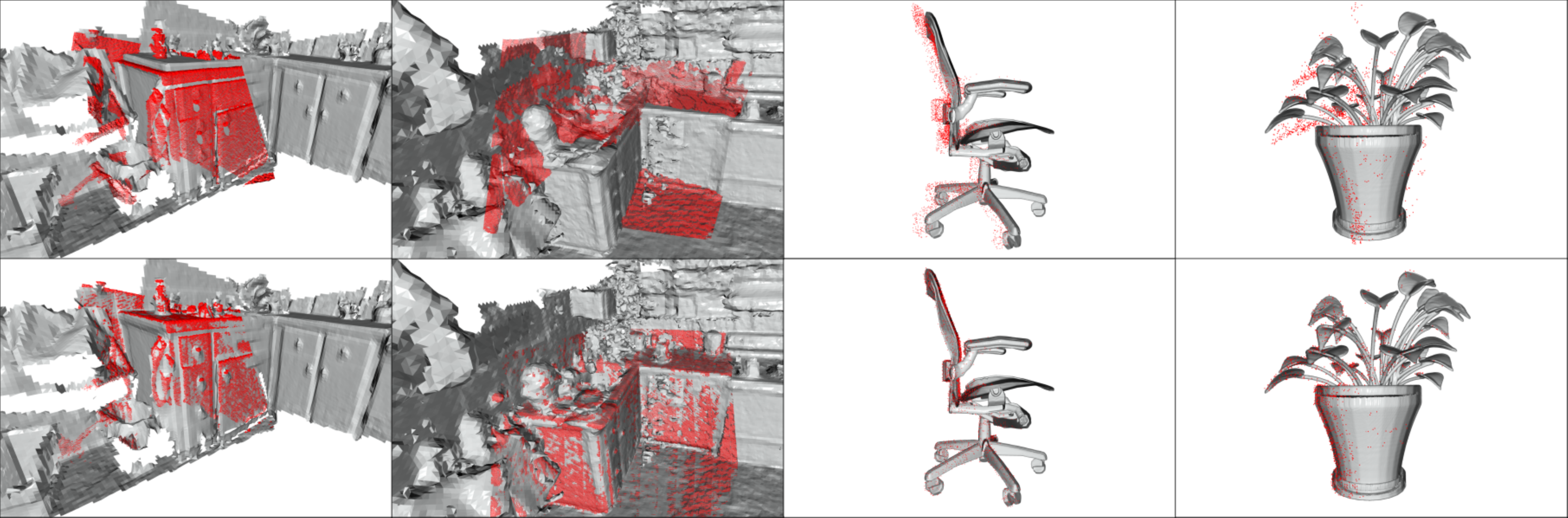}
\caption{Qualitative results of the TT-SDF2PC. On the top row, initial perturbed conditions of the PC w.r.t. the SDF. On the bottom, the registration results using TT-SDF2PC. On the left corresponds to 3DMatch~\cite{zeng20173dmatch} scenes ({\fontfamily{qcr}\selectfont analysis-by-synthesis-apt1-kitchen}, {\fontfamily{qcr}\selectfont bundlefusion-apt0}). On the right, 2 scenes from ModelNet \cite{modelnet} (chair, flower).}
\label{fig:qualitative}
\end{figure*}

\section{TT-Compression of Signed Distance functions}
\label{section:tttsdf}
Let us consider a closed 3D model $\Omega$ with surface $\partial\Omega.$ For any 3D point $x \in \mathbb{R}^3$ Signed Distance Function (Field) is defined as:
\begin{equation}
    SDF(x) = \left\{\begin{array}{rl}
        dist(x, \partial\Omega), &\textrm{if outside of object } \Omega \\
        -dist(x, \partial\Omega), &\textrm{otherwise.} \\
    \end{array}\right.
\end{equation}
Truncated SDF is defined as:
\begin{equation}
    TSDF(x) = \left\{\begin{array}{rl}
        \mu, &\textrm{if } \mu \leq SDF(x) \\
        SDF(x), &\textrm{if } -\mu < SDF(x) < \mu \\
        -\mu, &\textrm{if } SDF(x) \leq -\mu. \\
    \end{array}\right.
\end{equation}

Observations on behaviour of the SDF under low-rank compression were done in the TT-TSDF paper \cite{tttsdf}. Authors have demonstrated that TT provides $100\times$ - $1000\times$ compression ratio for volumetric TSDF scenes while keeping the metrics of error of the reconstructed surface at the order of $10^{-3}$ in Chamfer distance, which corresponds to barely visible distortion of small features of the surface. The resulting size of the representation is around 1-3MB for a high-resolution $512^3$ scene and 500-1500kB for a $256^3$ scene. It was shown that under TT compression TSDF faithfully keeps information in the close proximity of the surface of the object (zero isolevel). On the contrary, compressing the untruncated SDF preserves good functional approximation far from the surface of the object, however it was shown that the compression distorts or even completely ruins the surface information (see Fig. \ref{fig:grads}d).

A counter-intuitive result that we are going to show in sections below, is that even under strong compression (see Fig. \ref{fig:grads}e) the quality of gradients obtained through finite-difference approximation is still more than sufficient to provide very precise point-to-implicit registration.

\begin{figure*}[h!] 
\centering
\includegraphics[width=.98\textwidth]{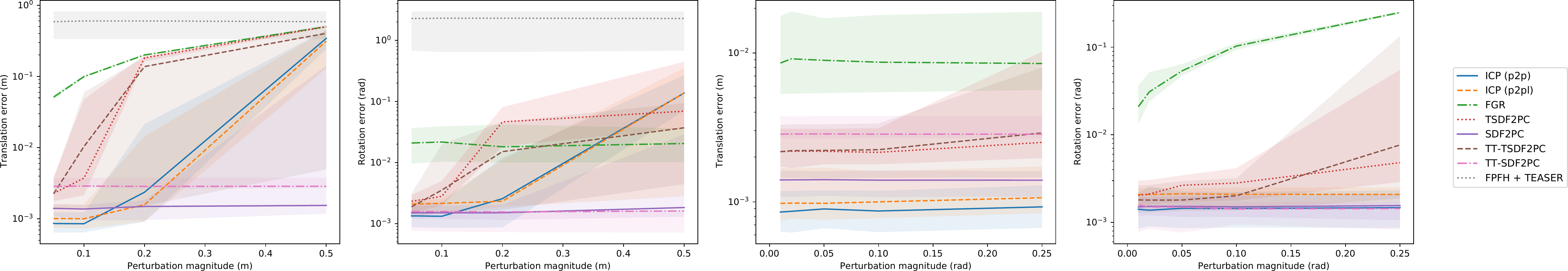}
\caption{Evaluation results of TT-SDF/TT-TSDF registration on ModelNet dataset. Reports median/Q1/Q3 values for both rotation and translation perturbations of different magnitudes.}
\label{fig:modelnet-res}
\end{figure*}

\begin{figure*}[h!] 
\centering
\includegraphics[width=.98\textwidth]{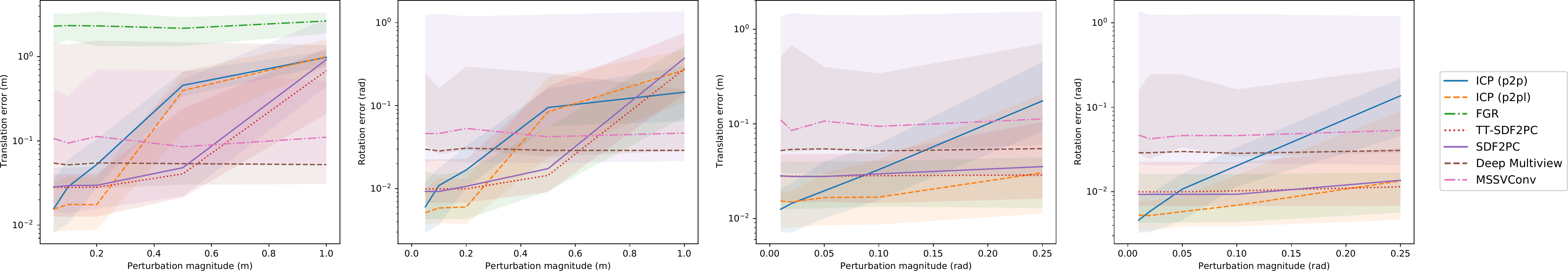}
\caption{Evaluation results of TT-SDF/TT-TSDF registration on 3DMatch dataset. Reports median/Q1/Q3 values for both rotation and translation perturbations of different magnitudes.}
\label{fig:3dmatch-res}
\end{figure*}

\section{Experiments}

\label{sec:exp}
In this section, we provide a comparison of the proposed registration method on compressed map TT-SDF2PC with respect to other popular registration methods from the accuracy and map memory consumption points of view. The study is performed on both synthetic and real datasets.

\subsection{Experimental setup}

All PC2PC methods are computed on a commodity Intel i7-7800X CPU, 64GB of RAM, and the deep learning based models and all SDF-based methods are inferred and computed on a GeForce GTX 1080 Ti 12Gb GPU.

\textbf{Data.} For evaluation, we consider two datasets: ModelNet \cite{modelnet}~--- object-centric dataset with 3D models of different objects and 3DMatch reconstruction dataset \cite{zeng20173dmatch}~--- meta-dataset of real RGBD data aggregated from SUN3D~\cite{xiao2013sun3d}, 7-scenes~\cite{shotton20137scenes}, BundleFusion~\cite{dai2017bundlefusion}, Analysis by Synthesis~\cite{valentin2016learning}, RGB-D Scenes v2 datasets~\cite{lai2014unsupervised}. The choice of ModelNet is motivated by its frequent usage in reporting new registration methods. 3DMatch is chosen as a wide collection of real depth data. In comparison to other raw 3D data, 3DMatch is well supported by tools for SDF/TSDF generation~\cite{whelan2015elasticfusion, KinectFusion2011, vizzo2022vdbfusion}. 

Every evaluation configuration has \textit{map} in two forms~--- point cloud and TSDF/SDF, a set of \textit{observations} to be registered w.r.t. this map. From ModelNet, 6 different objects are sampled, point cloud map and TSDF/SDFs are generated using object mesh. For each model, 5  different synthetic depth observations are produced. For 3DMatch, we pick one scene from each subdataset. Point cloud map for every scene is aggregated using every 10th point cloud from the sequence and further downsampled with voxel size \SI{0.05}{\metre}. TSDF map is obtained by using the VDBFusion~\cite{vizzo2022vdbfusion} algorithm with integration of every 10th frame into it and voxel size \SI{0.02}{\metre}. SDF is obtained by extracting isocontours from TSDF as occupancy grid and applying Euclidean distance transform to this occupancy grid.

In order to emulate the localization problem, we perturb every observation w.r.t its ground truth  pose in the map. The following sweep of perturbation magnitudes is considered: translation~--- [0.05, 0.1, 0.2, 0.5, 1.0] \SI{}{\metre}, rotation~--- [0.01, 0.02, 0.05, 0.1, 0.25] rad. For every perturbation magnitude, 5 directions are uniformly sampled. 

\textbf{Metrics.}
Our evaluation considers the following algorithm properties: registration accuracy and memory used to store the map. For accuracy, standard metrics over $SE(3)$ between ground truth and estimated transformation are taken, separately for rotation and translation parts. Memory for map storage is calculated according to the type of registration method: sparse point cloud map with $N$ points is presented as $N \times 3$ floats, SDF/TSDF~--- $W \times H \times L$ floats, where W, H, L are volume dimensions, TT-SDF~--- $W\times R + H \times R^2 + L \times R$ where $R$ is TT-rank, algorithm that operates on feature maps~--- $K \times D$ floats, where $K$~--- number of features in the map, $D$~--- map dimensionality.

\textbf{Methods.}
Our proposed registration methods on compressed map form are TT-SDF2PC and TT-TSDF (its truncated version). We consider the following methods to compare with. Among local registration methods, the classical ICP family (point2point, point2plane from the Open3D \cite{Zhou2018} implementation) and pc2SDF \cite{sdf2sdf2} algorithms are chosen. Among the global registration methods, we consider  Fast Global Registration (FGR)~\cite{zhou2018FGR}, TEASER with FPFH features~\cite{yang2020teaser}, and SOTA deep-learning features-based methods: MSSVConv~\cite{horache2021mssvconv} and Deep  Multiview~\cite{gojcic2020minkowsky}. The two last methods are used only on 3DMatch data using their original pre-trained versions for this dataset. For both of them we have chosen an  amount of features equal to 5000 because less amount produces less stable results. 

\subsection{Registration results}

The registration quality of the considered methods on ModelNet and 3DMatch datasets is depicted on Fig.~\ref{fig:modelnet-res} and \ref{fig:3dmatch-res} respectively. We do not include failed experiments in plots of average translation and rotation error metrics (i.e., TEASER on ModelNet and FGR on 3DMatch), since errors in non-converged cases can be arbitrarily large without carrying any significant information. 

In general, our solution TT-SDF2PC shows performance comparable on the order of error with {\em local} methods on small magnitudes and overcomes them on medium and large magnitudes. Such behaviour could be explained by the redundancy of data in SDF, where points far away from the zero level-set still provide a valid gradient towards the closest surface. Global registration methods, as  expected, are stable with respect to different types of perturbations but their overall performance presents a higher order error with respect to local methods.

SDF-variant of both classical and compressed registration methods outperforms TSDF-variant because the truncation regions or constant value lead to zero gradients. Originally, it was expected that the TT-SDF2PC variant would underperform the classical SDF2PC because of the information loss during compression. This hypothesis is supported by experiments on synthetic dataset, whereas on real data the behaviour is the opposite~--- the compressed form shows better performance than the raw SDF. This could be explained by the averaging effect of compression leading to an effective filtering effect of the noise in observations. 

Qualitative results on registration are presented in Fig.~\ref{fig:qualitative}. It could be noticed that even though TT-SDF2PC, with the nearest-voxel approximation of SDF, shows less minimal accuracy in comparison to other local registration, the  visual quality of alignment stays on high level and could be used for further applications in localization, registration and 3D reconstruction.  

\subsection{Memory Consumption of the methods}

Using metrics on measuring memory consumption, we provide an analysis for the maps used in 3DMatch datasets, whose statistics are presented in Table~\ref{table:stat}. Sparse point cloud map requires the least amount of data to be stored w.r.t. other approaches. Straightforward application of learnable feature-based approaches for global registration are not effective for being compressed form~--- less amount of features or their dimensionality leads to less and not-stable performance. TT-SDF representation decreases the amount of stored data in 10-200 times less than SDF depending on size of the scene. This achieves an order of compression of sparse point cloud on huge scenes (i.e., SUN3D), giving more stable registration accuracy on different size of magnitudes and providing highly-detailed surface information.

\begin{table*}[h!]
\caption{Memory requirements for map representation for different 3DMatch scenes. It contains the total number of points in the aggregated sparse point cloud map, the dimensions of the TSDF produced by VDBFusion~\cite{vizzo2022vdbfusion}. We also report the number of Bytes in memory for each map representation: PC, SDF, TT-SDF and feature-based approaches MMSVConv (5k) and Deep Multiview (5k).
}
\begin{center}
\begin{tabular}{l c c | c c c c c}
\toprule
Scene &   PC-map points  &  SDF dimensions & PC & SDF & TT-SDF  &  MSSVConv5k & Deep MV5k \\
\midrule
\text{red\_kitchen} & 26 757 & (346, 153, 152) & 320kB & 32.2MB & 3.1MB & 640kB & 640kB\\
\text{analysis-by-synthesis-apt1-kitchen} & 11 193 & (184, 210, 116) & 132kB & 17.9MB & 2.4MB & 640kB & 640kB \\
\text{bundlefusion-apt0} & 169 752 & (321, 295, 301) & 2.0MB & 114.0MB & 5.9MB & 640kB & 640kB\\
\text{rgbd-scenes-v2-scene\_01} & 23 690 & (265, 136, 265) & 284kB & 38.2MB & 2.8MB & 640kB & 640kB\\
\text{sun3d\_at-home\_at\_scan1\_2013\_jan\_1} & 407 266 & (736, 309, 934) & 4.9MB & 849.7MB & 6.5MB & 640kB & 640kB \\
\bottomrule
\end{tabular}
\end{center}
\label{table:stat}
\end{table*}

\subsection{Dependence on TT-ranks}
There is a single hyperparameter in the TT-based approach and it is the maximum TT-rank $R= \max(R_1, R_2)$. It is clear from the general fact about SVD that making the truncation rank smaller would remove high-frequency details from the data. 

To illustrate how the compression rank affects the quality of SDF we took the Stanford Armadillo model with a fixed starting pose perturbation and fixed simulated depth frame and performed a sweep on ranks (see Fig. \ref{fig:ranks}). As one can see, the registration quality almost saturate starting with $R=10$, which means that for the most part only large-scale geometry features of the SDF matter for pose optimisation.

\begin{figure}[h!] 
\centering
\includegraphics[width=.4\textwidth]{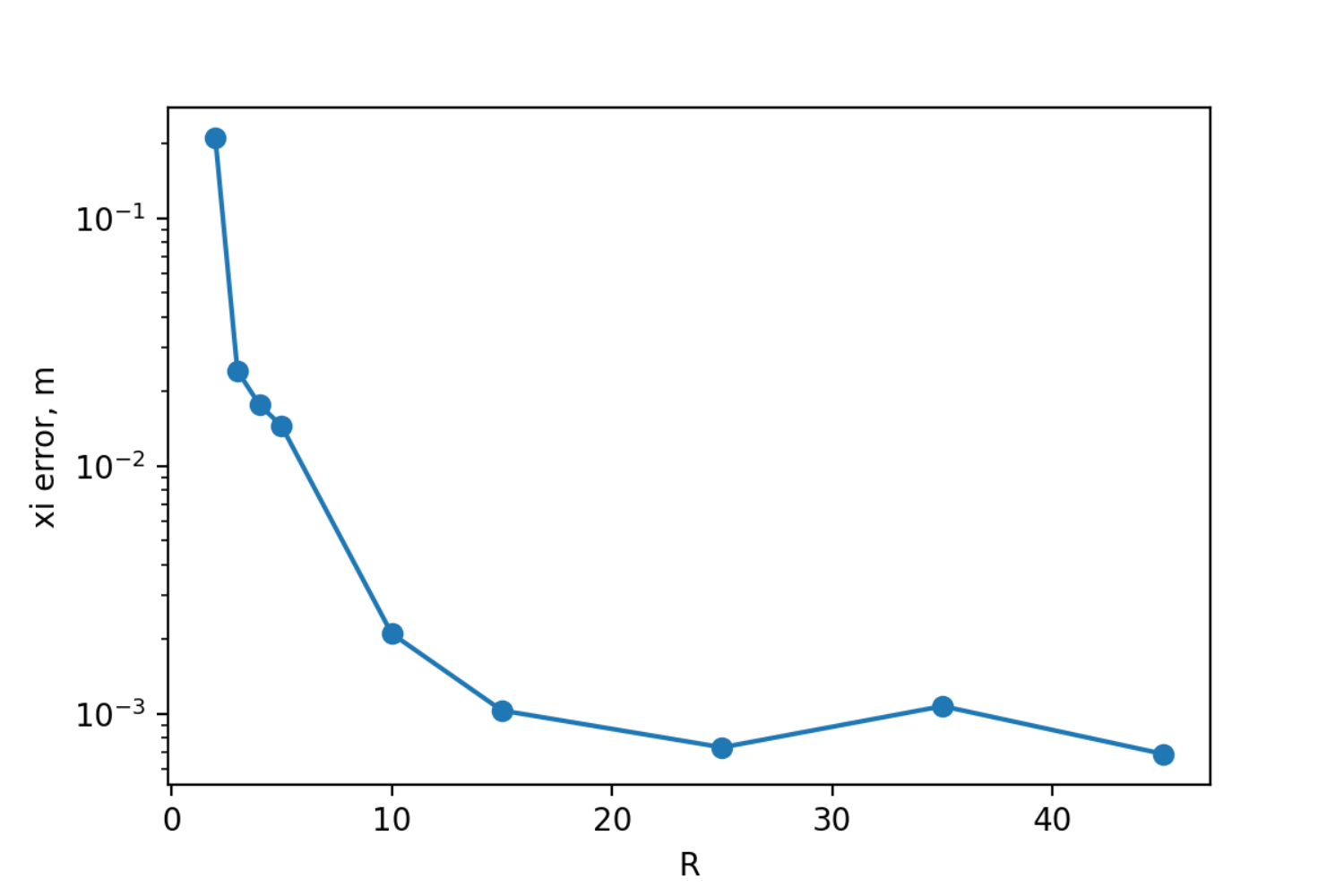} 
\includegraphics[width=.4\textwidth]{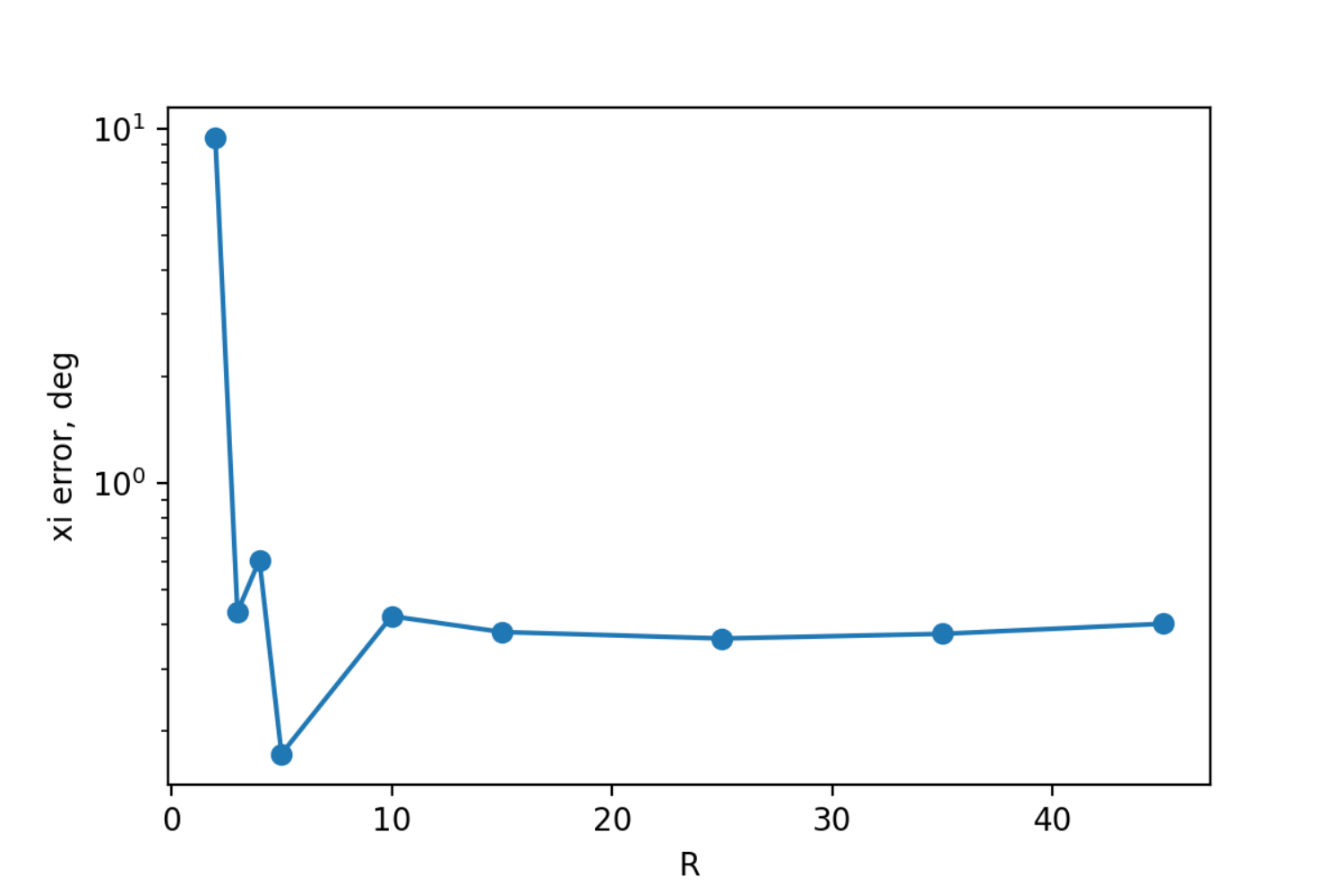}
\caption{The pose error in translation (up) and rotation (down), w.r.t the rank chosen for compression. }
\label{fig:ranks}
\end{figure}

\section{Conclusions}
Our method TT-SDF2PC for point cloud to compressed SDF registration provides up to millimeter accuracy on par with {\em local} registration methods and SDF-based methods in addition to a low memory footprint of a few megabytes. {\em Global} registration methods show better accuracy for larger perturbations, as expected, since our method is a local method.
This fulfills our two initial objectives, compression and the success of the registration task.

We have shown the potential of directly aligning PC in the tensor compressed domain on two different benchmarks. 
On the synthetic dataset (ModelNet) the method performs on ideal conditions, since the models are highly detailed and closed. 
On the real dataset, it is clear that there exists a gap between all the elements in the pipeline, observations, fusion, compression and registration; and yet, the results obtained bring support on the usage of TT-SDF2PC.

\end{document}